\def\year{2022}\relax
\documentclass[letterpaper]{article} 
\usepackage{aaai22}  
\usepackage{times}  
\usepackage{helvet}  
\usepackage{courier}  
\usepackage[hyphens]{url}  
\usepackage{graphicx} 
\usepackage{natbib}  
\usepackage{caption} 
\DeclareCaptionStyle{ruled}{labelfont=normalfont,labelsep=colon,strut=off} 
\usepackage{algorithm}
\usepackage{algorithmic}

\usepackage{xcolor} 
\usepackage{enumitem}
\usepackage{amsthm,amssymb,graphicx,multirow,amsmath,color,amsfonts}
\newtheorem{definition}{Definition}
\newtheorem{theorem}{Theorem}
\usepackage{caption}
\usepackage{subcaption}

\usepackage{newfloat}
\usepackage{listings}
\floatstyle{ruled}
\newfloat{listing}{tb}{lst}{}
\floatname{listing}{Listing}
\title{Two-Bit Aggregation for Communication Efficient and Differentially Private Federated Learning}
\author {
    Mohammad Aghapour\equalcontrib\textsuperscript{\rm 1},
    Aidin Ferdowsi\equalcontrib\textsuperscript{\rm 2},
    Walid Saad \textsuperscript{\rm 2}
}
\affiliations {
    \textsuperscript{\rm 1} Amirkabir University\\
    \textsuperscript{\rm 2} Virginia Tech\\
    m\_aghapour@aut.ac.ir, aidin@vt.edu, walids@vt.edu
}

\usepackage{bibentry}

\begin{document}

\maketitle

\begin{abstract}
In federated learning (FL), a machine learning model is trained on multiple nodes in a decentralized manner, while keeping the data local and not shared with other nodes. However, FL requires the nodes to also send information on the model parameters to a central server for aggregation. However, the information sent from the nodes to the server may reveal some details about each node's local data, thus raising privacy concerns. Furthermore, the repetitive uplink transmission from the nodes to the server may result in a communication overhead and network congestion. To address these two challenges, in this paper, a novel two-bit aggregation algorithm is proposed with guaranteed differential privacy and reduced uplink communication overhead. Extensive experiments demonstrate that the proposed aggregation algorithm can achieve the same performance as state-of-the-art approaches on datasets such as MNIST, Fashion MNIST, CIFAR-10, and CIFAR-100, while ensuring differential privacy and improving communication efficiency.


\end{abstract}
\section{Introduction}
Mobile phones, smart watches, and Internet of things (IoT) devices are examples of distributed networks that generate a wealth of data each day. The powerful computation capabilities of the devices within these networks and the importance of user data privacy led to the adoption of edge computation techniques where the data is localized and computations are done at the devices. Federated learning (FL) is a promising new framework for edge computing that enables distributed training of machine learning models while preserving user privacy. In FL, a large number of clients share their model parameters with a central server to learn a robust and comprehensive model without sharing their own training datasets \cite{mcmahan2017communication}.

FL typically consists of two main steps: (1) locally training models on edge devices with private local datasets and (2) aggregation of local models at the central server to come up with a more generalized model that can be used by all nodes. In centralized FL, a central server coordinates all the participating nodes during the learning process and aggregation of received model updates. Nodes send their updates to the server for aggregation, then all clients will receive aggregated new parameters. This approach prevents the nodes from transferring raw data to the server and improves privacy. The most challenging step in FL is aggregation within which the privacy of the local datasets can be attacked by reverse engineering the updates from the nodes. A significant advantage of FL is training a general model without need for direct access to the raw training data. The classical and standard method for aggregation is FEDAVG \cite{mcmahan2017communication}, in which the final weights are calculated by averaging element-wise on parameters of local devices. In \cite{pillutla2019robust}, the authors presented a robust aggregation approach to make FL robust to scenarios in which a fraction of the devices may be sending corrupted updates to the server. Their approach relies on an aggregation oracle based on the geometric median, which returns a robust aggregate using a constant number of calls to a regular non-robust secure average oracle.

In FL, because training processes take place at each client, the attack surface is limited to devices, instead of devices and the cloud, which makes it more secure and private \cite{mcmahan2017communication}. However, the recent work \cite{geiping2020inverting} shows that it is possible to reconstruct images at high resolution when having access to FL parameter gradients even for trained deep networks. In FL, private information can be extracted by analyzing the differences of updates from the clients. To prevent recovering raw data from trained weights, differential privacy (DP) is proposed in which, instead of sending trained parameters, agents incorporate some randomization into their shared updates with the central server. This randomization anonymizes the nodes and makes the revealing of the private raw data challenging for a potential adversary. DP algorithms mainly rely on adding a random noise into the updates, resulting in anonymous updates which makes it difficult to breach the privacy \cite{dwork2006calibrating}.
In \cite{wei2020federated}, the authors presented an FL framework based on DP, in which each client locally add noises to its trained parameters before uploading them to the server for aggregation, and they called it noising before model aggregation. In \cite{wang2019collecting}, the authors proposed local DP mechanisms for collecting a numeric attribute in contrast to global DP mechanisms. In local DP, each user modifies its information locally and only sends the randomized version to the server protecting both the users and the server from private information leaks. The authors in \cite{wu2020value} presented DP-based stochastic gradient descent (SGD) algorithms and studied their performance limits, which were shown to be linked to privacy settings and dataset sizes. The work in \cite{xie2018differentially} proposed a differentially private generative adversarial network (GAN) model, in which they achieved a higher privacy by adding a designed noise to the parameters of the model. The work in \cite{harder2021dp} proposed a DP data generation algorithm based on the random feature representation of kernel mean embeddings. When compared to GAN-based techniques, this method requires a considerably smaller privacy budget to create excellent data samples.
 
In addition to privacy challenges, the transmission of local updates to the central server can result in a huge communication overhead in the uplink causing delay and congestion in large scale FL scenarios. \cite{konevcny2016federated} proposed two approaches that address this challenge. One of them is structured updates, in which they learn an update directly from a restricted space parameterized by a smaller number of variables, such as low-rank or random masks. The other approach is sketched updates, in which they learn a full model update and then compress it using a combination of quantization, random rotations, and sub-sampling before sending it to the server. \cite{guo2020analog} used analogue over-the-air transmission to examine the analogue gradient aggregation approach for overcoming the communication bottleneck for wireless FL applications. The federated matched averaging (FedMA) method was introduced by \cite{wang2020federated} for FL of modern neural network architectures such as convolutional neural networks (CNNs) and long short-term memory (LSTM) networks. FedMA builds the shared global model layer by layer by matching and averaging hidden components (i.e. channels for convolution layers; hidden states for LSTM; neurons for fully connected layers) with comparable feature extraction signatures. FedMA not only outperforms popular state of the art federated learning algorithms on deep CNN and LSTM architectures trained on real-world datasets, but it significantly decreases total communication overhead, according to their findings.

However, the works in \cite{dwork2006calibrating,wei2020federated,wang2019collecting,wu2020value,xie2018differentially,harder2021dp} improve the privacy of the nodes with a cost of a loss in the model performance. 
Using DP to improve privacy generally decreases algorithm performance, and, hence, there is a key trade-off between privacy and convergence performance in the training process. In addition, these methods guarantee global DP which may fail to preserve privacy in some scenarios. Furthermore, the approaches described in \cite{mcmahan2017communication,pillutla2019robust,guo2020analog,wang2020federated} that are proposed to overcome the communication inefficiency of the FL algorithms fail to provide privacy guarantees.

The main contribution of this paper is a novel federated aggregation process that is communication efficient and differentially private. In our proposed method, the node updates will contain only two bits for each parameter of the learning model; one bit representing the absolute value and one for the sign of each parameter. Moreover, we propose a new aggregation mechanism that combines the received two-bit updates from the nodes and returns updated parameters to the nodes. The proposed update and aggregation mechanisms reduce the communication overhead significantly on the uplink depending on the binary representation (BR) of the model parameters. Moreover, we prove that the proposed mechanism provides close to absolute DP especially with large number of bits used in the BR of the parameters. Our experiments show that the proposed mechanism has similar performance to the state of the art FL works on the known datasets such as CIFAR-10, MNIST, and MNIST-fashion while reducing the communication overhead and providing DP for nodes. For instance, for a 32-bit BR, the communication overhead reduces by a factor of 1/16 and the DP can be improved by 1 order of magnitude.

\section{System Model and Problem Formulation}
Let $\mathcal{N}$ be a set of $n$ nodes such that each node $i$ has a local datasets $\mathcal{D}_i$. In practice, a portion of the nodes may not be available, however, in our analysis we consider that all the nodes are available for the entire process of learning. Let $ f $ be the global parameterized machine learning (ML) model that each node aims at learning on its own dataset. In FL, each agent $ i $ trains the model $ f $ on its own dataset $ \mathcal{D}_i $ locally to learn the parameters $ \boldsymbol{w}_i $ of $ f $ and after every $ e $ epochs of training it shares an update $\boldsymbol{U}_i$ about its model with the central server. After receiving all the updates from the nodes, the server aggregates the updates using an aggregation function $ h $ and returns a new parameter set $ \boldsymbol{w} $ to the nodes. Formally, the new parameter set at the $ k $-th step of the aggregation can be written as:
\begin{equation}\label{eq:aggregation}
	\boldsymbol{w}^{k+1} = h\left(f,\boldsymbol{w}^k , \boldsymbol{U}^k_{i| \forall i \in \mathcal{N}}\right).
\end{equation}
The goal of the server is to design an aggregator mechanism $ h $ in \eqref{eq:aggregation} such that the model after convergence will have a higher performance than a standalone case in which the nodes train the model on their local dataset without participating in FL. Moreover, the $\boldsymbol{U}_i$ for the nodes should be designed such that the communication overhead on the uplink is reduced and the local data privacy of the nodes is preserved. 

Prior FL works such as \cite{pillutla2019robust,guo2020analog,wei2020federated,wang2019collecting,wu2020value} use an additive noise on either the model parameters or the training gradients to design a privacy preserving mechanism. However, the size of the transmitted data is the same as the size of the model parameters, making those FL solutions communication-inefficient. In addition, those algorithms do not guarantee a good level of privacy because the local data may be regenerated by the updates transmitted from the nodes \cite{geiping2020inverting}. In the following, we propose a novel update and aggregation mechanisms that provides uplink communication efficiency as well as differential privacy while not impacting the performance of the FL process.

\section{Two-Bit Aggregation Approach}
We propose a novel two-bit aggregation algorithm that does not share the actual parameters of the local models with the central server hence it preserves the privacy of local devices. Similar to FedAVG each node performs its own computation based on $ f $ and its own dataset $\mathcal{D}_i$, however, the update and aggregation steps are different.
We propose a mapping $\mathcal{M}$ that, after each FL iteration, i.e. after every $e$ epochs, maps each local gradient of the model $g_j$ into a two bit BR, $\boldsymbol{u}_j$, and transmits it to the central server. Thus, the update vector $\boldsymbol{u}_j$ contains the two bit representations of each gradient. In contrast to the state-of-the-art FL algorithms in which every gradient value is represented by a $p$-bit binary format where $p$ can typically be 16, 32, or 64, in our proposed algorithm, regardless of the BR of the values in the computation stage each gradient is mapped into a to a 2-bit BR. This reduces the uplink communication overhead by $p/2$ times compared to the other FL algorithms. For instance, using our proposed algorithm for a 64-bit BR in computation stage, the size of the $\boldsymbol{u}_i$ is reduced by 32 times. In addition, this approach improves the privacy of the nodes as a more abstract information is transmitted from each local node to the server. In fact, we prove that our proposed algorithm provides guarantees differential privacy. In what follows, we explain the details of the proposed mapping algorithm.
\subsection{Two-Bit Mapping Algorithm}
The proposed mapping of gradients to a two-bit BR has two main steps: scaling and bit selection. Although scaling comes before bit selection, we will go over bit selection first for clarity. After each FL step, when the new gradients are available at the local nodes, the local nodes transmit a two bit vector $\boldsymbol{u}$ for each gradient to the server. The first bit represents the sign of the gradient while the second bit is derived from the absolute value of the gradient. In order to derive this bit, first, absolute value of each gradient is converted to a fixed-point BR $\boldsymbol{b}$. The node then attaches the $l$-th bit of the BR to the sign bit and transmits these two bits for each gradient to the server, where $l$ is the server's requested location of the bit in the BR. Later we will define how the server chooses the value of $l$. We also note that, in the fixed point BR, number of bits for integer and fraction part of the number is fixed. In our algorithm, we assume that out of $p$ bits $p-1$ of them are for integer part and 1 bit is for the sign. This essentially means that no fractions are taken into account for the BR of the values. However, as more training epochs pass, the absolute value of the gradients reduces. Although this is not a problem in a floating point BR, it does pose some difficulties in a fixed-point BR. When using a fixed-point BR, the BR of a gradient uses fewer bits out of available bits as its absolute value decreases. To address this issue at each FL iteration, the server informs the nodes about the predicted maximum absolute value of the next iteration. In other words, the server assumes in the next iteration the absolute value of the calculated gradients will not exceed a certain value $m$. Thus, knowing $m$, each node will first multiply the calculated gradient by $ 2^{p-1} / m $ since $ 2^{p-1} $ is the largest value that can be constructed by $ p-1 $ bits. If the resulting value is greater than $ 2^{p-1} $, we use $ 2^{p-1} $. Note that, $m$ is updated at every FL iteration based on the gradients received from the nodes. The steps of the mapping algorithm are summarized in Figure \ref{fig:mapping_fig} and Algorithm \ref{alg:Mapping_alg}.

\begin{figure}
	\centering
	\includegraphics[width=\columnwidth]{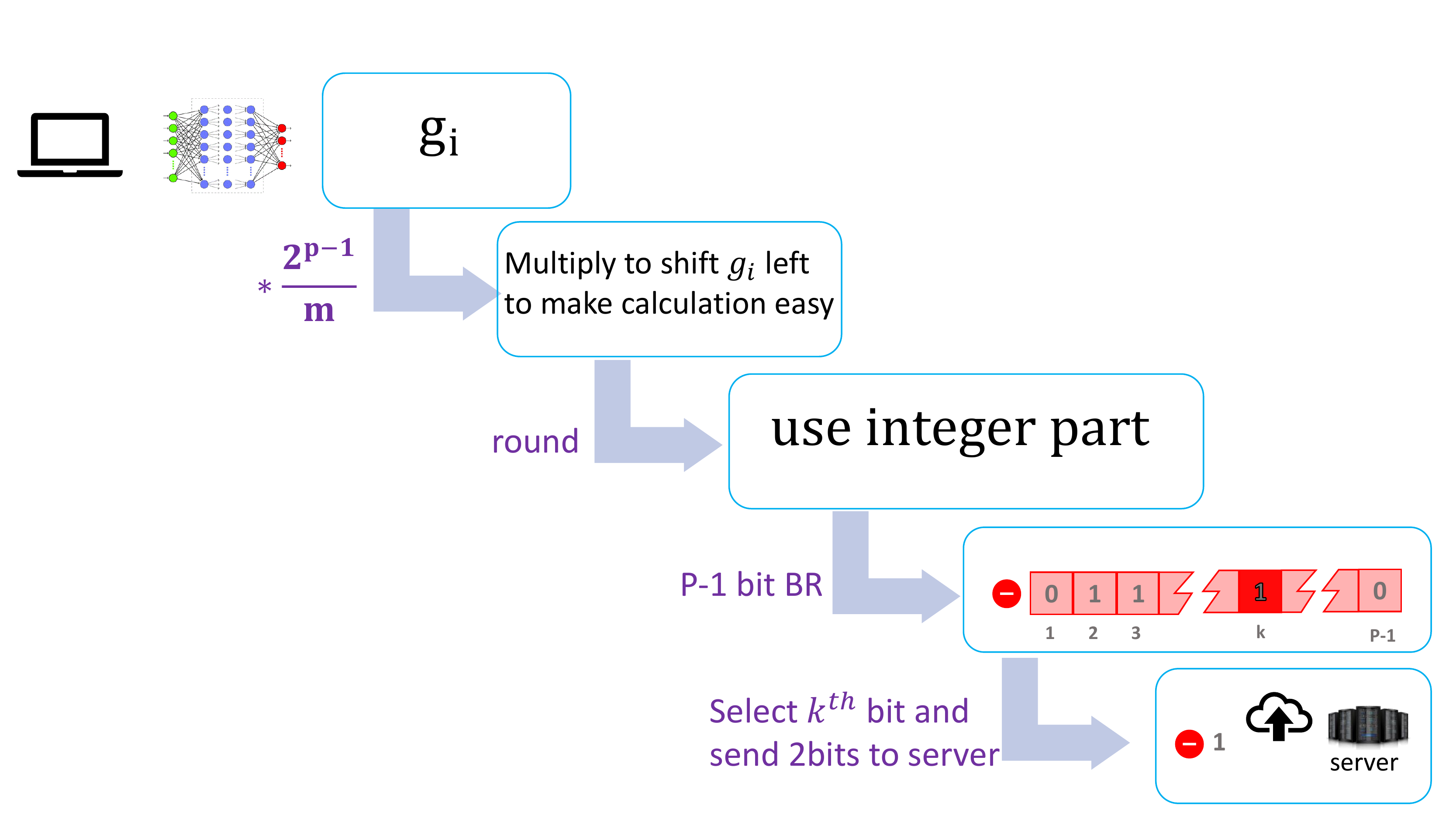}
	\caption{Mapping Process}
	\label{fig:mapping_fig}
\end{figure}

As mentioned, the proposed two bit mapping significantly reduces the uplink communication overhead. In what follows, we also prove that the proposed mapping also provides differential privacy. To this end, first we formally define differential privacy.

\begin{definition}
	(Differential Privacy\cite{dwork2006calibrating,dwork2006our}) A randomized algorithm $ \mathcal{M} $ is $ (\epsilon) $-differentially private if for all neighboring BRs $\boldsymbol{b}$ and $\bar{\boldsymbol{b}}$, and for all sets $ \mathcal{F} $ of outputs,
	\begin{equation}
		\textrm{Pr}\left[ \mathcal{M}(\boldsymbol{b}) \in \mathcal{F} \right] \leq exp(\epsilon).\textrm{Pr}\left[ \mathcal{M}(\bar{\boldsymbol{b}}) \in \mathcal{F} \right]
	\end{equation}
	The probability is taken over the random coins of $ \mathcal{M} $.
\end{definition}
Next, we prove that our proposed mapping is differentially private.
\begin{theorem}
	The proposed two-bit mapping $ \mathcal{M} $ is $ (\ln \frac{p}{p-2}) $-differentially private.
\end{theorem}
\begin{proof}
	Let $\boldsymbol{x}$ and $\bar{\boldsymbol{x}}$ be two neighboring $p$-bit BR of gradients. In this case, only one of the bits at locations 2 to $p$ differs for $\boldsymbol{x}$ and $\bar{\boldsymbol{x}}$. Without loss of generality, let the bit at location $p$ be different, then we will have:
	\begin{equation}
		\boldsymbol{x}=(b_1,b_2,b_3,...,b_p) , \bar{\boldsymbol{x}}=(b_1,b_2,b_3,...,\bar{b}_p)
	\end{equation}
	Since $ b_1 $ is the sign bit, the output of the mapping $ \mathcal{M} $ will be either $ (b_1,0) $ or $ (b_1,1) $. Thus, next we will compare the probability of these two outputs for $ \boldsymbol{x} $ and $ \bar{\boldsymbol{x}} $.
	Assuming $ b_p = 1 $, the probability of $ \mathcal{M}(\boldsymbol{x}) = (b_1,0) $ can be written as follow:
	\begin{align}
		\textrm{Pr}(\mathcal{M}(\boldsymbol{x}) = (b_1,0)) &=\underbrace{ \frac{1}{p-1} \times \frac{1}{2} + ... + \frac{1}{p-1} \times \frac{1}{2}}_{p-2 \textrm{ terms}} \nonumber\\\label{eq:epsilonprivate}
		&= \frac{p-2}{(p-1)\times 2}
	\end{align}
	In \eqref{eq:epsilonprivate}, the probability of selecting any non-sign bit in $ \boldsymbol{x} $ is $ \frac{1}{p-1} $ and the probability of selected bit being 0 is $ \frac{1}{2} $. In this case, since $ \textrm{Pr}(\bar{b}_p =0)=1 $, then $ \textrm{Pr}(\mathcal{M}(\bar{\boldsymbol{x}}) = (b_1,0)) = \frac{p-2}{(p-1)\times 2} + \frac{1}{p-1} $. Therefore we will have
	\begin{equation}
		\frac{\textrm{Pr}(\mathcal{M}(\boldsymbol{x})=(b_1,0))}{\textrm{Pr}(\mathcal{M}(\bar{\boldsymbol{x}})=(b_1,0))}=\begin{cases}
			\frac{p-2}{p}, & \text{if $b_p=0$}\\
			\frac{p}{p-2}, & \text{if $b_p=1$}
		\end{cases}
	\end{equation}
	Therefore,
	\begin{equation}
		\frac{\textrm{Pr}(\mathcal{M}(\boldsymbol{x})=(b_1,0))}{\textrm{Pr}(\mathcal{M}(\bar{\boldsymbol{x}})=(b_1,0))} \leq \frac{p}{p-2}
	\end{equation}
	By following similar steps, we can also show that 
	$\frac{\textrm{Pr}(\mathcal{M}(\boldsymbol{x})=(b_1,1))}{\textrm{Pr}(\mathcal{M}(\bar{\boldsymbol{x}})=(b_1,1))} \leq \frac{p}{p-2}$.
	Therefore, we have
	\begin{equation}
		\textrm{Pr}(\mathcal{M}(\boldsymbol{x}))\leq \frac{p}{p-2}\textrm{Pr}(\mathcal{M}(\bar{\boldsymbol{x}})),
	\end{equation}
	which means $\mathcal{M}$ is a $ (\ln \frac{p}{p-2}) $-differentially private.
\end{proof}
Theorem 1 shows that the proposed 2-bit mapping is differentially private. In addition, we can see from Theorem 1 that, as the size of the BR ($p$) increases, the algorithm becomes more private since $\ln \frac{p}{p-2}$ goes to zero. Next, we show how the server chooses the location of the bit from the BR to be transmitted by local nodes.

\begin{algorithm}[tb]
	\caption{Mapping algorithm}
	\label{alg:Mapping_alg}
	\textbf{Input}:  $ g_i , m$\\
	\textbf{Output}: $\boldsymbol{U}_i$
	\begin{algorithmic}[1] 
		\STATE Multiply all $ g_i $ to $ \frac{2^{p-1}}{m} $ 
		\STATE Represent integer part of all $ g_i $ by $ p-1 $-bit binary format.
		\FOR{each weight of model $ f $}
		\FOR{each node}
		\STATE $ \boldsymbol{u}_i $ value is combination of $ i $-th node's sign value and server selected bit's value.
		\ENDFOR
		\STATE Generate $ \boldsymbol{U}_i $ by concatenating $ \boldsymbol{u}_i $ values
		\ENDFOR
		\STATE \textbf{return} $\boldsymbol{U}_i$ 
	\end{algorithmic}
\end{algorithm}

\subsection{Bit Selection and Aggregation at The Server}

\begin{algorithm}[tb]
	\caption{Aggregation algorithm}
	\label{alg:aggregation_alg}
	\textbf{Input}:  $\boldsymbol{U}_i$\\
	\textbf{Output}: $ \boldsymbol{w} $
	\begin{algorithmic}[1] 
		\FOR{each weight of model $ f $}
		\IF {number of nodes greater than $ p-1 $}
		\STATE cluster nodes which sent same bit together and to have $ p-1 $ clusters.
		\ENDIF
		\FOR{each cluster}
		\STATE Separate positive and negative values
		\FOR{each bit}
		\STATE Do majority voting to find final bit value (unfilled bits replace with zero).
		\ENDFOR
		\STATE Place bits in their corresponding locations and make a positive and a negative $ p-1 $-bit number.
		\STATE re-scale these numbers to their real value by multiplying them to $ \frac{m}{2^{p-1}} $
		\STATE new weight value calculates by weighted averaging these two numbers.(their weight amount is the number of positive or negative bits used for each value)
		\ENDFOR
		\STATE update $ m $ value to the maximum amount between of $ m $ and absolute value of these two number 
		\ENDFOR
		
		\STATE \textbf{return} $ \boldsymbol{w},m $
	\end{algorithmic}
\end{algorithm}

\begin{figure}[]
	\centering
	\includegraphics[width = \columnwidth]{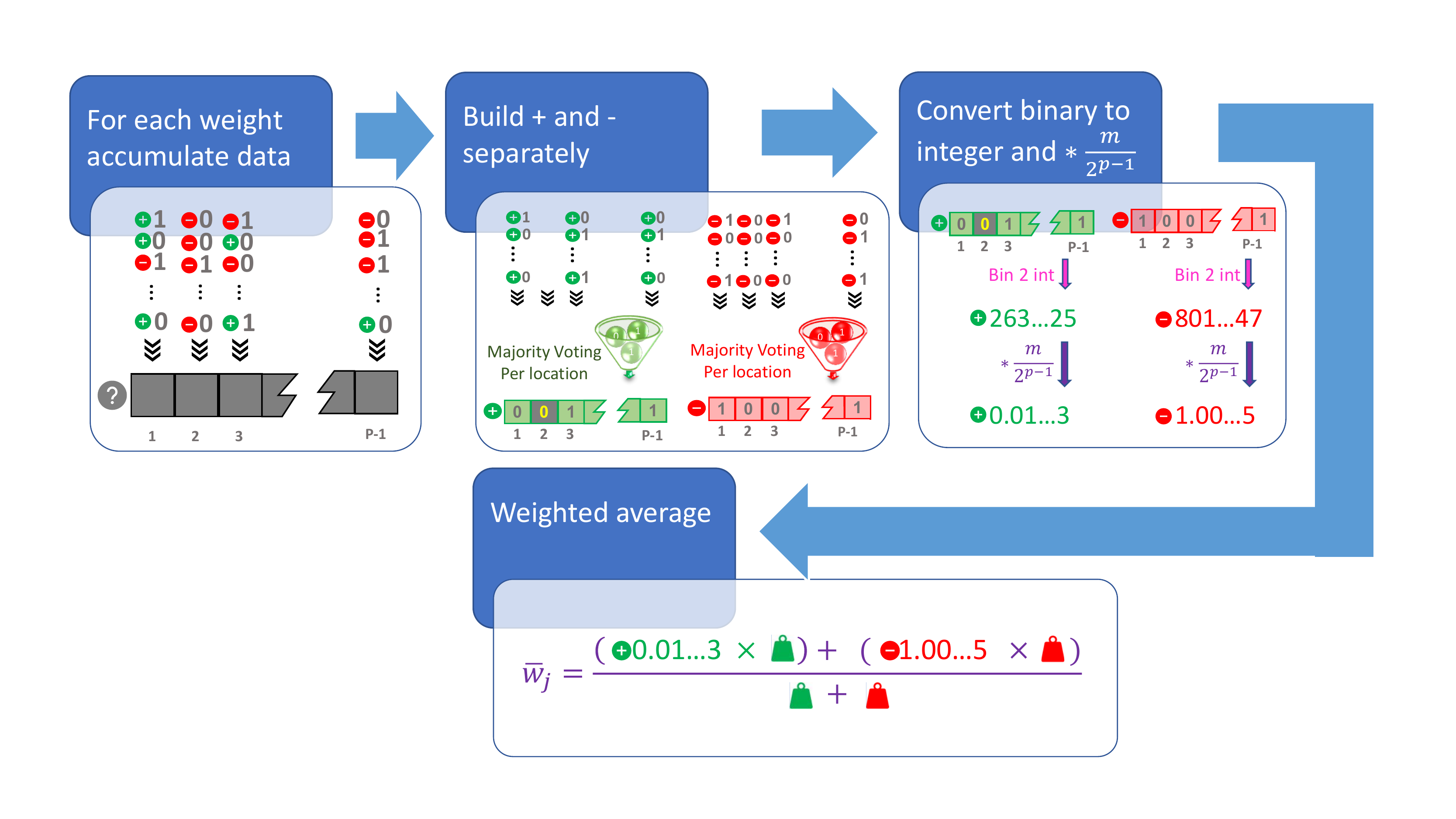}
	\caption{Aggregation Process}
	\label{fig:Aggregation_fig}
\end{figure}

Assuming $n\geq p-1$, at every FL aggregation round, the server randomly assigns the location of a bit in the BR to a local node such that each bit is assigned at least to one local node. Then, at each FL iteration, besides the updated weights of the model, the server transmits the requested bit location to the local nodes. Each local node will use the requested bit location for the mapping $\mathcal{M}$ of the next FL iteration. Also, in order to reduce the downlink communication overhead, the server only sends one location request to each node rather than sending one location for every parameter of the model. Then, each node $i$ increments the location by 1 for each row of the update $\boldsymbol{U}_i$. For example if the server's requested location is $l$, then the node uses $l$ for the first row of the update $\boldsymbol{U}_i$, uses $l+1 $ for the second row of the update $\boldsymbol{U}_i$, and so on. This increases the randomization and helps in a better convergence based on our experiments.

\begin{figure*}[t]
	\centering
	\subfloat[MNIST]{\includegraphics[width=3.2cm, angle =90]{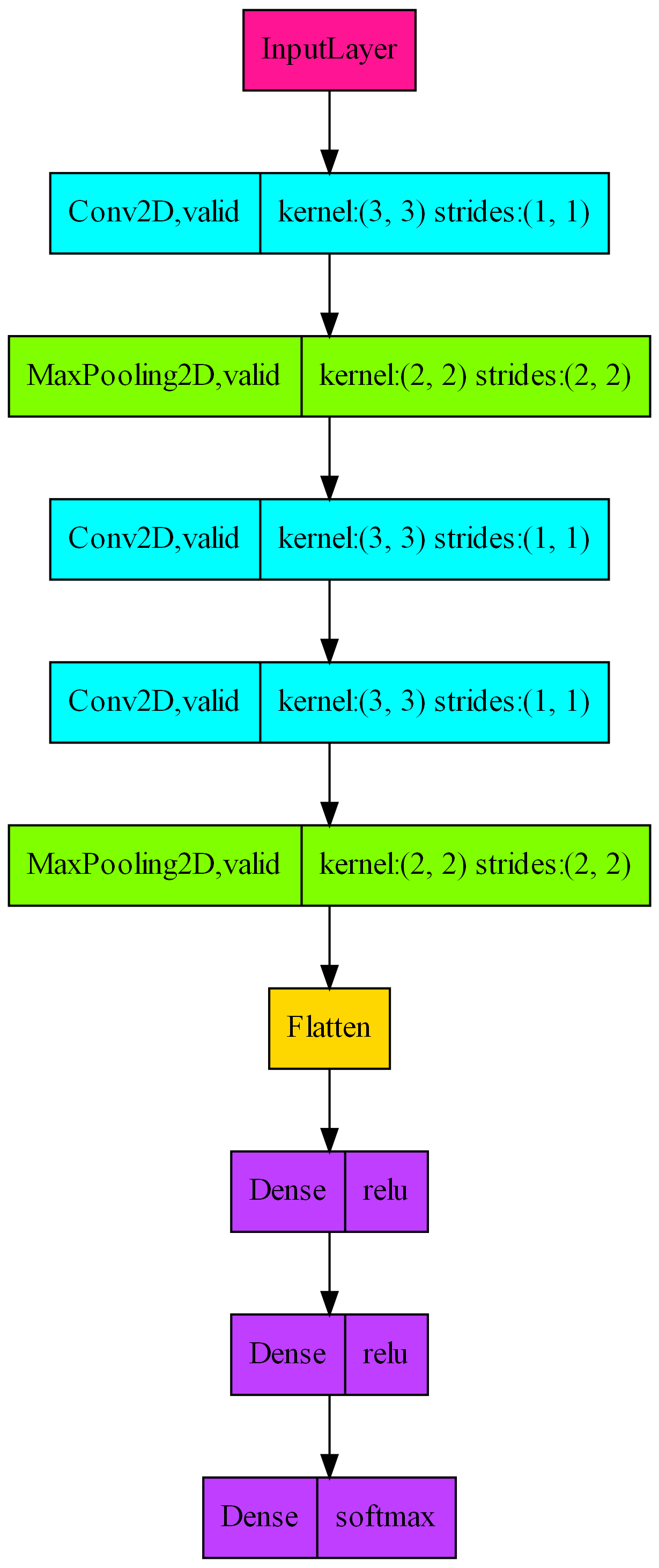}}
	\vspace{0.00mm}
	\\
	\subfloat[Fashion MNIST]{\includegraphics[width=3.2cm, angle =90]{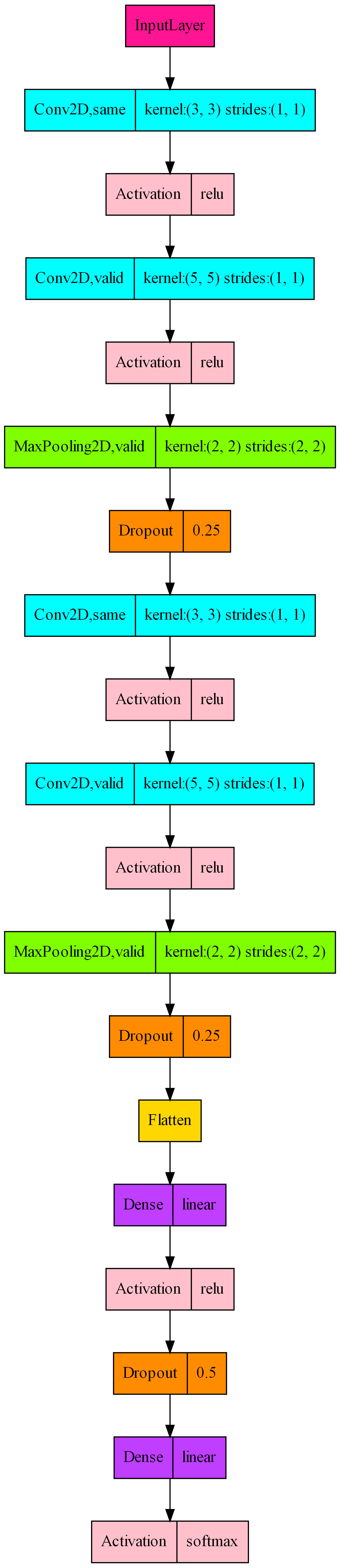}}
	\vspace{0.00mm}
	\qquad
	\subfloat[CIFAR 10]{\includegraphics[width=3cm, angle =90]{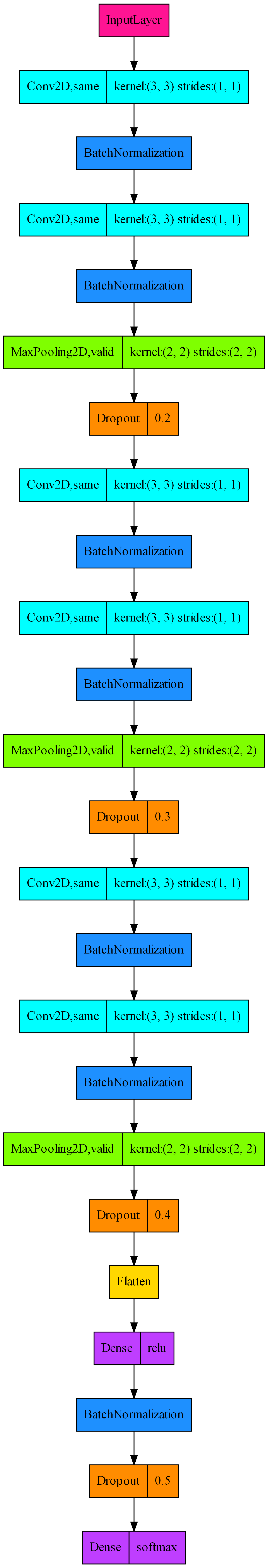}}
	\vspace{0.00mm}
	\qquad
	\subfloat[CIFAR 100]{\includegraphics[width=3cm, angle =90]{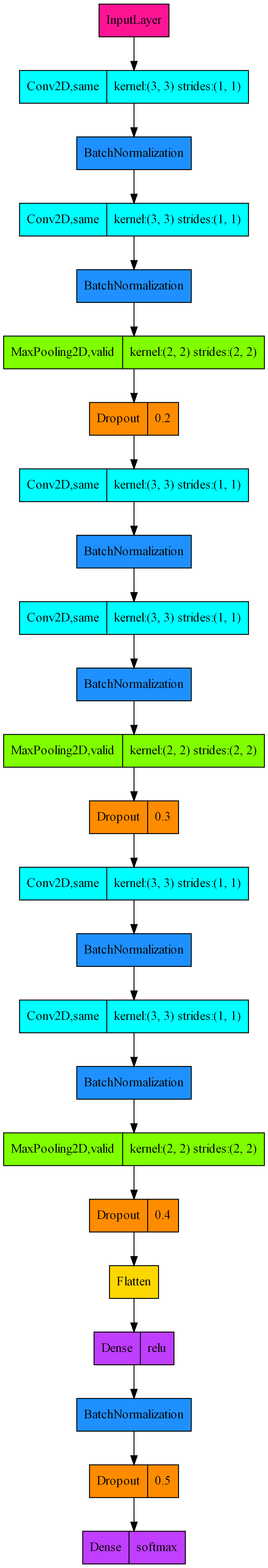}}
	\caption{Architectures used for training on each dataset.}
	\vspace{0.00mm}
	\label{dataset}
\end{figure*}

For aggregation, the server first groups the received updates with the same bit location into two groups based on their sign bit. In other words, every update $\boldsymbol{u}_{j}$ from nodes that were assigned the same bit location is clustered into the negative group if its first bit is 0 and is cluster into positive group otherwise. Next, if more than one node exists in a group, the server uses a majority voting to decide what the non-sign bit should be. For some locations either negative or positive group may not exist. In this case we consider the bit for the non existing group of that location to be zero. After finding the bits for every location, the server puts them in their location and generates a negative $p-1$-bit BR based on negative groups and a positive $p-1$-bit BR based on the positive groups. Next, after converting these two BR to decimal values we get the weighted average of the derived negative and positive values where the weight of negative and positive numbers are proportional to the number of nodes with negative and positive updates. Finally, we multiply the weighted average by $ m / 2^{q-1}$.  After server runs these steps on every $\boldsymbol{u}_{j}$, it updates the value of $m$ by finding the maximum absolute value of the updates. Figure \ref{fig:Aggregation_fig} and Algorithm \ref{alg:aggregation_alg} summarize the bit selection and aggregation algorithms.

By following the two steps of the proposed two-bit aggregation, i.e., mapping and aggregation steps, the nodes are able to preserve their local data privacy and improve the uplink communication efficiency. In the following, we present our experiments on different datasets and FL scenarios.

\section{Experimental Results and Analysis}

In our experiments, we set $ n $=31, $ e $=10, and $ p $=32 unless stated otherwise. We split the data into training and test sets with a 80/20 ratio.  For each node, we distribute the training dataset by dividing the number of training data by the number of nodes $ n $. However, we use the same test dataset for all of the nodes.

\subsection{Datasets and Model Architectures}

For our experiment, we employed various well-known datasets, such as MNIST, Fashion MNIST, CIFAR-10, and CIFAR-100, with state of the art methods to evaluate our proposed Two-bit algorithm. Figure \ref{dataset} shows the architecture of the models used for each dataset. We have used 3 Tesla P-100 GPUs for training of our models.

\subsection{Performance Comparison}

\begin{table}[]
	\begin{center}
		\caption{Comparison of achieved accuracy of the proposed aggregator versus other solutions.}
		\label{tab:table1}
		\begin{tabular}{|c|c|c|c|c|}
			\hline
			&
			MNIST &
			\begin{tabular}[c]{@{}c@{}}Fashion\\ MNIST\end{tabular} &
			\begin{tabular}[c]{@{}c@{}}CIFAR\\ 10\end{tabular} &
			\begin{tabular}[c]{@{}c@{}}CIFAR\\ 100\end{tabular} \\ \hline
			Standalone                                                         & 95.74\%          & 83.52\%          & 55.34\%          & 18.28\%          \\ \hline
			FedAVG                                                             & 98.83\%          & 85.66\%          & 64.21\%          & 38.62\%          \\ \hline
			\begin{tabular}[c]{@{}c@{}}DP-enabled \\ FedAVG\end{tabular} & 98.85\%          & 85.29\%          & 64.11\%          & 37.90\%          \\ \hline
			\textbf{Two-bit}                                                   & \textbf{98.83\%} & \textbf{85.42\%} & \textbf{64.23\%} & \textbf{38.33\%} \\ \hline
		\end{tabular}
	\end{center}
\end{table}
First, we compare our proposed mechanism with FedAVG \cite{mcmahan2017communication}, FedAVG with DP \cite{mcmahan2018general}, and a standalone agent with no FL. For a fair comparison we have adjusted the parameters of the additive noise in DP-enabled FedAVG to yield a similar $\epsilon$ as our method. Table \ref{tab:table1} shows the average accuracy of the models after convergence. As shown in Table 1, the proposed two-bit aggregation always achieves almost the same accuracy as the other FL methods. This shows that the proposed framework will not impact the FL process in terms of performance while providing DP and communication efficiency.

\subsection{Communication Efficiency}

\begin{figure}[t!]
	\centering
	\includegraphics[width=\columnwidth]{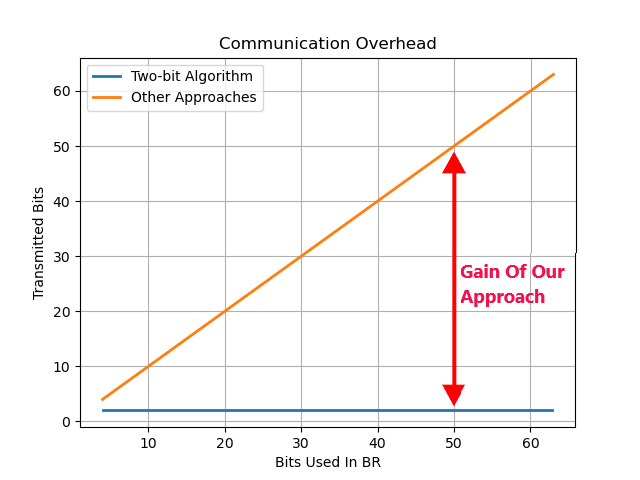}
	\caption{Comparison of Proposed Two-bit aggregation with other approaches in terms of communication overhead}
	\label{fig:communication_overhead}
\end{figure}

Figure \ref{fig:communication_overhead} shows a comparison of the per iteration communication overhead of our proposed 2-bit aggregator versus the other FL algorithms. As we can see from Figure \ref{fig:communication_overhead}, for any BR size, $p$, our proposed algorithm always requires transmitting 2 bits per parameter, however, all the other FL algorithms' communication overhead increases linearly with respect to the number of bits used in the BR. This is a clear benefit of our proposed mechanism over the other methods: Our approach can improve the communication efficiency by the factor of $p/2$ compared to other FL solutions.

\subsection{Differential Privacy Efficiency}

\begin{figure}[t!]
	\centering
	\includegraphics[width=\columnwidth]{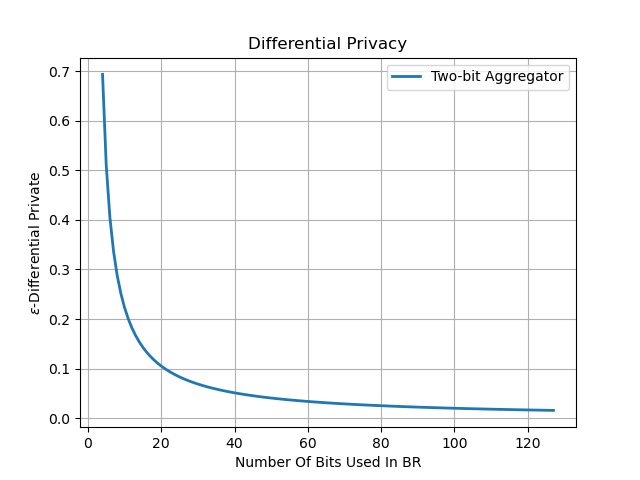}
	\caption{Impact of number of bits used in BR on Differential Privacy}
	\label{fig:dp}
\end{figure}

Figure \ref{fig:dp} shows how our proposed aggregator's $\epsilon$ decays when we increase the number of bits in BR. The reason for this is that by increasing the size of the BR, the randomized selection algorithms perform better, and our transmitted bit is more anonymized, resulting in a more privacy-preserving algorithm. From this figure, we observe that, even for a case where $p=4$, i.e., only 4 bits used for BR, we achieve a DP with $\epsilon = 0.69$. This is a very important result since if we want to achieve such level of DP for the DP-enabled FedAvg in \cite{mcmahan2018general}, we need to compensate in $\delta$ of the DP\footnote{Please refer to \cite{dwork2006calibrating} for the definition of $(\epsilon,\delta)$ DP.}. For example, for $\epsilon = 0.69$ in DP-enabled FedAVG, we need a $\delta \approx 1.1$. The value of $\delta$ increases as $\epsilon$ decreases, however, our proposed aggregation will always yield $\delta =0$ for all $\epsilon$ definitions. Therefore, as observed from Figure 5, we can increase the privacy of the aggregation by increasing the number of bits used in BR. Note that, this will not increase the communication overhead while improving the privacy - a feature that no other FL approach can provide. 

\subsection{Convergence Rate Analysis}

\begin{figure*}[t]
	\centering
	\begin{subfigure}[b]{0.45\textwidth}
		\includegraphics[width=\textwidth]{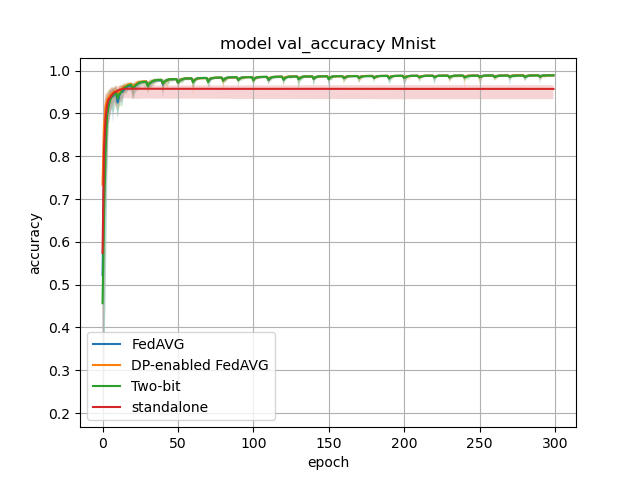}
		\caption{MNIST}
		\label{fig:Mnist}
	\end{subfigure}
	~ 
	\begin{subfigure}[b]{0.45\textwidth}
		\includegraphics[width=\textwidth]{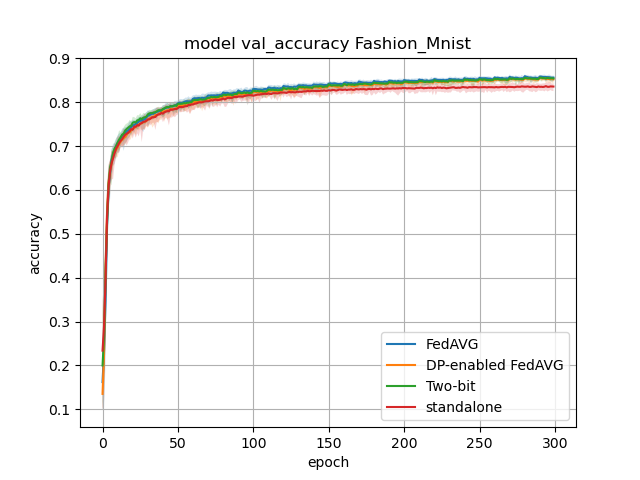}
		\caption{Fashion MNIST}
		\label{fig:Fashion_Mnist}
	\end{subfigure}
	~ 
	\begin{subfigure}[b]{0.45\textwidth}
		\includegraphics[width=\textwidth]{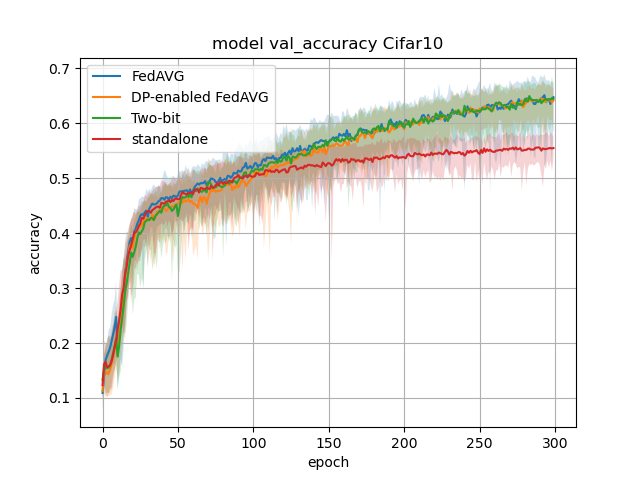}
		\caption{CIFAR10}
		\label{fig:Cifar10}
	\end{subfigure}
	~ 
	\begin{subfigure}[b]{0.45\textwidth}
		\includegraphics[width=\textwidth]{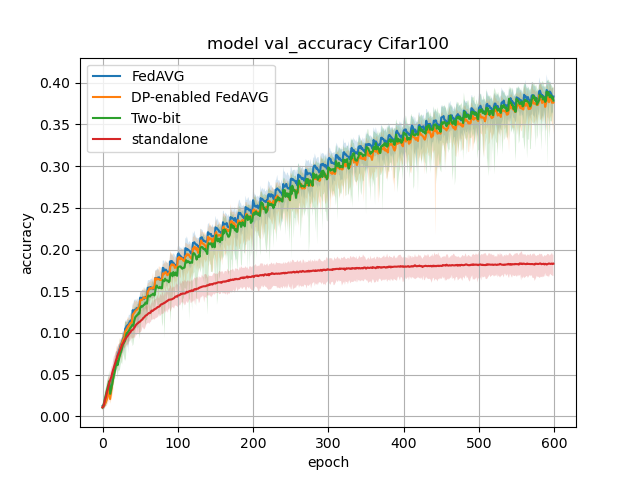}
		\caption{CIFAR100}
		\label{fig:Cifar100}
	\end{subfigure}
	\caption{experiment results}\label{3fig}
\end{figure*}

Figure \ref{3fig} compares our proposed method with FedAVG, DP enabled FedAVG, and standalone case in terms of convergence rate. As you can see from Figure \ref{3fig}, our 2-bit aggregation mechanism has almost same convergence rate compared to the FedAVG and DP-enabled FedAVG. This is an outstanding result because, despite the fact that 2-bit aggregation only transmits two bits of data per iteration, it requires the same number of communication rounds as the other FL methods to converge. This means that the 2-bit aggregator not only reduces the instantaneous communication overhead but also in the long term it has requires less communication resources than the other FL methods. Figure \ref{3fig} also indicates that any FL algorithm (including our proposed method) outperforms the stanadlone agent that trains its local model only on its local data.

\subsection{Large Scale Systems}
\begin{figure}[]
	\centering
	\includegraphics[width=\columnwidth]{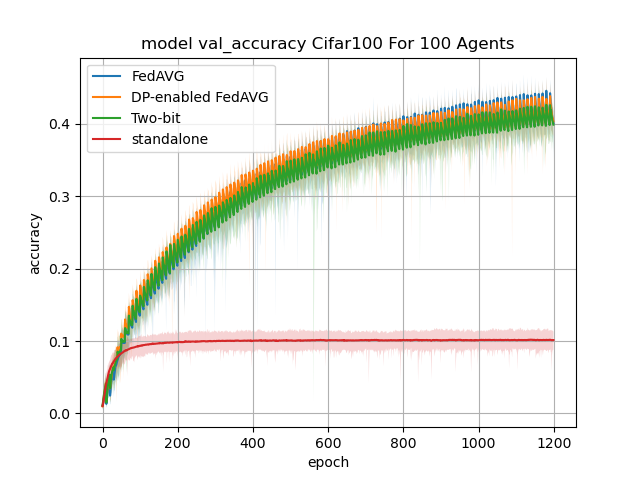}
	\caption{CIFAR100 dataset with 100 nodes}
	\label{fig:cifar100_100agents}
\end{figure}
We also have applied our proposed mechanism to a large scale FL scenario. In Figure \ref{fig:cifar100_100agents}, we have considered 100 nodes with $p=32$. In this case the total number of samples for CIFAR-100 is equally shared between 100 nodes. (excluding the test set which is equal for all of the nodes.) From Figure \ref{fig:cifar100_100agents}, we observe that the proposed 2-bit aggregation yields almost the same performance as the other FL mechanisms, while reducing the communication overhead. The communication overhead becomes a crucial parameter in large scale FL scenarios. For example, if there was a limitation in the communication overhead, our proposed method can incorporate $p/2$ times more nodes than the other FL methods since we can save $p/2$ times in the communication overhead.

\section{Conclusion}
In this paper we have proposed a novel FL mechanism to improve data privacy, and reduce uplink transmission overhead. In this approach each node transmits only two-bits per parameter to the central aggregator instead of a whole BR value, resulting in a reduced communication overhead and increased privacy level. Experimental results also confirm that our proposed aggregation mechanism achieves similar accuracy and convergence rate and a higher privacy level on MNIST, Fashion MNIST, CIFAR-10, and CIFAR-100 compared to other FL approaches.

\newpage
\bibliography{references}
\end{document}